%% file: peng_153.tex
\documentclass[accepted]{uai2023} 

\usepackage{microtype}
\usepackage{graphicx}
\usepackage{subfigure}
\usepackage{booktabs} 

\usepackage{colortbl}

\usepackage{amsmath,amsfonts,amsthm}
\usepackage[ruled,linesnumbered]{algorithm2e}
\usepackage{comment}
\usepackage{wrapfig}
\usepackage[]{caption}
\usepackage{xcolor}         
\usepackage{minitoc}

\definecolor{alizarin}{rgb}{0.82, 0.1, 0.26}
\hypersetup{linkcolor=alizarin}
\hypersetup{citecolor=[rgb]{0.0,0.0,0.95}}

\theoremstyle{plain}
\newtheorem{theorem}{Theorem}[section]

\theoremstyle{definition}

\theoremstyle{remark}

\usepackage[american]{babel}

\usepackage{natbib} 
\usepackage{mathtools} 
\usepackage{booktabs} 
\usepackage{tikz} 

\title{Copula for Instance-wise Feature Selection and Ranking}

\author{Hanyu Peng, Guanhua Fang, Ping Li\\
Cognitive Computing Lab\\ Baidu Research\\
No.10 Xibeiwang East Road, Beijing 100193, China\\
10900 NE 8th St. Bellevue, Washington 98004, USA\\
{\tt \{hanyu.peng0510, fanggh2018, pingli98\}@gmail.com}
}

\begin{document}
\maketitle

\begin{abstract}
  Instance-wise feature selection and ranking methods can achieve a good selection of task-friendly features for each sample in the context of neural networks. However, existing approaches that assume feature subsets to be independent are imperfect when considering the dependency between features. To address this limitation, we propose to incorporate the Gaussian copula, a powerful mathematical technique for capturing correlations between variables, into the current feature selection framework with no additional changes needed. Experimental results on both synthetic and real datasets, in terms of performance comparison and interpretability, demonstrate that our method is capable of capturing meaningful correlations.
\end{abstract}
\section{Introduction}
The primary goal of feature selection is to select the most relevant features, and simultaneously help benefit the downstream tasks like recognition~\citep{nilsson2007consistent,stanczyk2015feature}, clustering~\citep{dash2000feature,liu2005toward,witten2010framework}. Traditional approaches view feature selection and downstream tasks as two separate entities. However, recent methods based on neural networks have achieved greater performance by combining the two, resulting in a differentiable optimization. Furthermore, the importance of features can vary from sample to sample; one set of features may be useful for recognizing one thing, while a completely different set of features may be necessary to identify something else. Therefore, it is necessary for feature selection to be done on a case-by-case basis, which is known as instance-wise feature selection~\citep{yoon2019invase,chen2018learning,masoomi2020instance}, it \textit{selects unique features} for each sample.

Characteristics in reality are often correlated, such as a person's height and weight, economic status and wealth, etc. The identification of feature correlations can \textit{minimize the redundancy of features}. Yet, in the literature of instance-wise feature selection and ranking methods~\citep{chen2018learning,yoon2019invase,abid2019concrete,masoomi2020instance,wu2018tracing} that \textit{follow the context of neural networks}, the dependencies between features has not been considered manifestly. For instance, L2X~\citep{chen2018learning} performs a feature selection for maximizing the mutual information between selected feature subsets and corresponding outputs. Lowering the KL distance between the selected features and all features, is the guideline of INVASE~\citep{yoon2019invase}. However, little or even no consideration in modeling correlations can pose a challenge to tasks.

To address this matter, we \textit{underscore the importance of capturing the pairwise relationship between features}, and explicitly model it through a probabilistic framework. This paper examines both binary feature selection and top-$k$ feature ranking. Specifically, for the top-$k$ feature ranking problem, the number of features chosen is assumed to be known and unchanging. However, for binary feature selection, the number of features varies across samples, as different samples may have different feature significance. For both scenarios, two sampling policies are presented in the paper, along with a corresponding neural network implementation to unite feature selection and task learning.

\textbf{Our contributions} are summarized as follows:
\begin{itemize}

\item We explore the explicit dependence between features through copula, such as instance-wise feature selection and top-$k$ ranking, as the initial step. 

\item Two sampling schemes, implemented via neural networks, have been carefully crafted to ensure accuracy and efficiency. Moreover, they have been rigorously tested and verified to guarantee their efficacy.

\item 
The experimental results have been found to be highly indicative and of superior quality, as evidenced by metrics such as accuracy, true positive rate (TPR), and false discovery rate (FDR). Indeed, these results demonstrate a remarkable explanatory power.
\end{itemize}

\section{Preliminaries and Related Work}

\textbf{Notation}: 
Lowercase typeface letters (${x}$) represent scalar, lowercase bold typeface letters ($\boldsymbol{x}$) represent vectors, uppercase typeface letters ($X$) stand for random variable, uppercase bold typeface letters ($\boldsymbol{X}$) represent matrix.

\subsection{Feature Selection}

In recent decades, feature selection has evolved from a specialized field to a commonplace technology in the machine learning community~\citep{guyon2003introduction,koller1996toward,liu2007computational}, and has been employed to simplify systems and make them more interpretable to researchers~\citep{guyon2007causal,brown2012conditional}. Existing works on instance-wise feature selection can be divided into two main categories: binary feature selection~\citep{chen2018learning,yoon2019invase} and top-$k$ feature ranking~\citep{abid2019concrete,yamada2020feature}. The distinction between them is that the number of binary feature selection is dependent on the instance, while the number of top-$k$ feature ranking is user-defined and known. Our work is akin to~\citet{yoon2019invase} and~\citep{chen2018learning}, in that all have a 3-stage network structure. ~\cite{sokar2022pay} is based on sparse autoencoders and uses a new sparse training algorithm to quickly attend to informative features, but does not provide theoretical analysis or justification for the algorithm. The second method, called Stochastic Gates~\cite{yamada2020feature}, selects a small subset of features based on probabilistic relaxation of the `0 norm, and simultaneously learns a non-linear regression or classification function. However, it does not provide a direct differentiable top-k feature selection method and does not explore the correlation between features in a differentiable manner. Nevertheless, none of the aforementioned works have explicitly modeled the dependency; instead, we propose incorporating a Gaussian copula into binary feature selection and top-$k$ feature ranking.

\subsection{Reminders on Copula}

 Copula~\citep{cherubini2004copula,trivedi2007copula,jaworski2010copula} is a powerful tool for describing the correlation between variables by modelling their joint distributions, given the known marginal distributions where each variable follows the uniform distribution $[0,1]$. It can be combined with a variational auto-encoder (VAE) to address the issue of posterior collapse in the latent space~\citep{zizhuangwang2019neural}. On the other side,~\cite{wang2020relaxed} introduces a new distribution called the relaxed multivariate Bernoulli distribution (RelaxedMVB), which is a  reparameterizable relaxation of the multivariate Bernoulli distribution. The RelaxedMVB combines the Gaussian copula and the
Relaxed Bernoulli to create a continuous relaxation of
the multivariate Bernoulli distribution. Furthermore,~\citet{suh2016gaussian} proposed to employ Gaussian copula to model the local dependency. Additionally, copula has been extensively applied in finance applications~\citep{cherubini2004copula,cherubini2011dynamic}. 

We briefly review the copula theory. For any continuous random variables, $Y_1, \ldots, Y_d$, let their marginal distributions be $F_1(y_1)=P(Y_1 \leq y_1), \ldots, F_d(y_d)=P(Y_d \leq y_d)$ correspondingly. 
Then it is easy to see that $U_i := F_i(Y_i)$ follows the uniform distribution on $[0,1]$ for $i \in [d]$.
And the copula of $U_1, \ldots, U_d$ is defined as the joint cumulative distribution function (C.D.F) of $(U_1, \ldots, U_d)$.
\begin{eqnarray}
C(u_1, \ldots, u_d) := P(U_1 \leq u_1, \ldots, U_d\leq u_d), ~~ U_i = F(Y_i) \nonumber 
\end{eqnarray}
is a function from $[0,1]^d$ to $[0,1]$. We can rewrite the above formula as:
\begin{eqnarray}
C(u_1, \ldots, u_d) := P(Y_1 \leq y_1, \ldots, Y_d\leq y_d), \nonumber
\end{eqnarray}
One of the most helpful copula is known as the Gaussian copula, which is constructed from a multivariate normal distribution over $ \mathbb{R}^d$ by using the probability integral transform.
Its explicit formula is given as 
\begin{eqnarray}\label{eq:copula:def}
C_{gaussian}(u_1, \ldots, u_d; \boldsymbol R) 
= \Phi_{ \boldsymbol R}(\Phi^{-1}(u_1), \ldots, \Phi^{-1}(u_d)), \nonumber
\end{eqnarray}
where $\Phi_{\boldsymbol R}(\cdot)$ is the joint cdf of multivariate normal distribution with mean zero and correlation/covariance matrix $\boldsymbol R  \in \mathbb R^{d \times d}$, $\Phi^{-1}(\cdot)$ is the inverse of a set of marginal Gaussian cdf. Off-diagonal elements in $\boldsymbol R$ capture  pairwise relation between different marginals.

\section{Method}\label{sec:method}

In this section, we shall commence by presenting the problem statement. Thereafter, we shall introduce two sampling schemes for binary feature selection and top-$k$ feature ranking, respectively. Subsequently, we shall provide the overall algorithm workflow. Lastly, we shall discuss the details of implementing neural networks in greater detail.

\subsection{Problem Formulation}

Given a collection of samples $\{(\boldsymbol{x}^{n},{y}^{n})\}_{n=1}^N$, where $\boldsymbol{x}^{n}=\left(x_1^{n}, \ldots, x_d^{n}\right)$ is a $d$-dimension feature vector, ${y}^{n}$ is the corresponding output. Feature selection seeks to select a subset of features and simultaneously learn a task-specific objective function under certain loss metrics as follows:
\begin{align*}
 \label{eq:fs-form}
\min_{\boldsymbol{\theta}} &~ \frac{1}{N} \sum_{n = 1}^{N} \mathcal{L} \left[ \Phi\left( \boldsymbol{\theta}; \boldsymbol{x}^{n} \odot \boldsymbol{z}^{n} \right), {y}^{n} \right] , \\
\text{subject to}&\quad \boldsymbol{z}^{n}=\mathcal S \left(\boldsymbol \alpha^{n} \right); \boldsymbol \alpha^{n} = f\left(\boldsymbol{w},\boldsymbol{x}^{n}\right), 
\end{align*}
where $\boldsymbol{\theta},\boldsymbol{w}$ denote the learnable parameters, $ \boldsymbol \alpha^{n}$ represents the learned score, $\odot$ represents the element-wise product, $\boldsymbol{z}^{n}=\left(z_{1}^n,\ldots,z_{d}^n \right)\in \{0,1\}^d$ stands for the feature indicator, $\mathcal{L}$ is the objective function, $\mathcal S(\cdot)$ represents the sampling function, $f \left( \boldsymbol{w}, \cdot \right)$ and $\Phi \left( \boldsymbol{\theta}, \cdot \right)$ stand for the mapping function to learn $ \boldsymbol \alpha^{n}$ and infer prediction, respectively. For top-$k$ feature ranking, since the specific size of active features is established, we add an \textit{auxiliary constraint} $\big\| \boldsymbol{z}^{n} \big\|_{0} = k, \forall n=1\cdots N$ to loss function. In contrast, the number of binary feature selection is sample-dependent, the \textit{regularization term} $\frac{\lambda}{N} \sum_{n=1}^N \big\| \boldsymbol{z}^{n} \big\|_{1}$ that controls selected features number, where $\lambda$ is the trade-off parameter.

The above expression shows that the selection of features is dictated by the variable $\boldsymbol{z}^{n}$. Also on the other hand, $\boldsymbol{z}^{n}$ is depended on $\boldsymbol \alpha^{n}$. Thus the final performance is greatly affected by the way how we model $\boldsymbol \alpha^{n}$. Due to expressive power and complexity of neural network, we naturally introduce neural network e.g. multilayer perceptron (MLP) to learn the mapping $\boldsymbol \alpha^{n}=f\left(\boldsymbol{w},\boldsymbol{x}^{n}\right)$.

Despite we have known how to parameterize $\boldsymbol \alpha^{n}$, sampling function $\mathcal{S}(\cdot)$ usually is non-differentiable, as non-differentiable methods may have a large variance. It will hinder the usage of standard back-propagation algorithm and has large variance. Besides, $\mathcal{S}(\cdot)$ mostly assumes elements of $\boldsymbol{z}^{n}$ to be independent, like Gumbel-Softmax~\citep{jang2017categorical}. To address the first issue, we develop two sampling schemes as a continuous differentiable approximation to Bernoulli distribution and top-$k$ ranking without replacement. To address the second issue, we incorporate the copula function to model the dependency between features explicitly. In what follows, \textit{we drop the superscript $n$ for simplicity}.

\subsection{Sampling Scheme via Binary Mask}

When the number of features is unknown and sample-dependent, for each $i \in \{1,\ldots, d\}$, we define a binary-valued random variable $z_i$ which indicates whether the $i$-th feature should be included in $\mathcal S_a$.  Specifically, $z_i$ follows this distribution:
$P(z_i = 1) = \frac{\exp\{\alpha_i\}}{1 + \exp\{\alpha_i\}}$
and $p(z_i = 0) = \frac{1}{1 + \exp\{\alpha_i\}}$.
Here $\alpha_{i}$'s are obtained from weight layer.

By Gumbel-Max trick, we know that $z_i$ and $\hat z_i$ have the same distribution, where $\hat z_i = 1\{g_{i} + \alpha_i > 0\}$
and $g_i=\log\frac{u_i}{1-u_i}$ follow a standard logistic distribution. $u_i$'s follow the uniform distribution $[0,1]$, and each element is independent to the other during the generation procedure in the classical Gumbel-Max algorithm. Also, as we can see, $\hat z_i$ is an indicator function and hence is not differentiable with respect to $\alpha_i$.
To circumvent this issue, we consider to replace $\hat z_i$ by its soft counterpart, that is, 
$\tilde z_i = \frac{1}{1 + \exp\{- (g_i + \alpha_i) / t\}}$. 
Here $t$ is a tuning parameter. When $t \rightarrow 0$, it is not hard to see that $\tilde z_i \rightarrow \hat z_i$.

In addition to characterizing the marginal distribution, we need to take into account the dependence among different $z_i$'s.
It suffices to model the joint distribution of $g_i$'s.
To do so, we incorporate a copula to accommodate its dependence structure.
We write the joint C.D.F. of $(g_1, \ldots, g_d)$ as 
$G(x_1, \ldots, x_d) = P(g_1 \leq x_1, \ldots, g_d \leq x_d)$. 
We assume the following parameterized form,
$G(x_1, \ldots, x_d) = C_{gaussian}(L(x_1), \ldots, L(x_d); \boldsymbol R)$, 
where $C_{gaussian}$ is a Gaussian copula with $\boldsymbol R$ as its correlation matrix and $L(x) = \frac{1}{1 + \exp\{-x\}}$.
In a summary, $\boldsymbol{\tilde z} \sim G(x_1, \ldots, x_d)$ ($\boldsymbol{\tilde z} = (\tilde z_1, \ldots, \tilde z_d)$) is the output of the feature selection~layer.

A natural question is how to guarantee the correlation matrix to be positive definite.
We consider the following parameterization scheme. 
From the perspective of factor analysis~\citep{harman1976modern, kline2014easy}, 
suppose a random vector $\boldsymbol x = (x_1, \ldots, x_d)$ lies in a lower-dimensional manifold:
\begin{eqnarray*}
\boldsymbol x = \boldsymbol L \boldsymbol \xi + \boldsymbol \epsilon,
\end{eqnarray*}
where $\boldsymbol \xi$ is a latent random $p$-vector and $\boldsymbol L = (\boldsymbol l_1, \ldots, \boldsymbol l_p)$ is the coefficient/loading matrix with size being $d \times p$.
Suppose $\boldsymbol \xi$ and $\boldsymbol \epsilon$ are mutually uncorrelated, covariance matrix of $\boldsymbol \xi$ is an identity and the noise level of $\boldsymbol \epsilon$ is $\sigma^2$, then
we know 
\begin{eqnarray*}
\boldsymbol \Sigma = \boldsymbol L \boldsymbol L^T + \sigma^2 \boldsymbol I.
\label{eq:sigma_noise}
\end{eqnarray*}

In other words, we can parameterize the correlation matrix as 
$\boldsymbol R = \textrm{Norm}(\boldsymbol L \boldsymbol L^T + \sigma^2 \boldsymbol I)$,
where function $\textrm{Norm}(\cdot)$ maps a covariance matrix to a correlation matrix and satisfies
$(\textrm{Norm}(\boldsymbol \Sigma))_{ij} = \boldsymbol \Sigma_{ij}/ (\boldsymbol \Sigma_{ii} \boldsymbol \Sigma_{jj})^{1/2}$.
By adopting this parametrization, it's guaranteed that the correlation matrix is positive definite. Procedure to obtain correlated noise via Gaussian copula is summarized in Algorithm~\ref{algo:corr_noise}, algorithm flow for binary feature selection in Algorithm~\ref{algo:sample_binary_mask}. The proposed sampling scheme is first proposed in~\cite{wang2020relaxed} and our sampling scheme is identical to this.
\begin{algorithm}.  
\caption{Generate Correlated Uniform Noise via Gaussian Copula}
\label{algo:corr_noise}
\KwIn{Full-rank or low-rank matrix $\boldsymbol L$, identity matrix $\boldsymbol{I}$ and noise level $\sigma^2$.}
\KwOut{Correlated noise $\boldsymbol u$.}

Obtain the covariance matrix via low-rank approximation $\boldsymbol \Sigma =\boldsymbol L^T \boldsymbol L + \sigma^2 \boldsymbol I$ or full-rank approximation $\boldsymbol \Sigma =\boldsymbol {L}^T \boldsymbol L$ \;

Perform Cholesky factorization on $\boldsymbol \Sigma$ to get Cholesky factor $\boldsymbol V$\;

Generate a Gaussian noise vector $\boldsymbol \zeta$ from standard normal distribution $\boldsymbol{\zeta} \sim \mathcal{N}(0,\,\boldsymbol{I}_d )$\;

Get the correlation matrix $\boldsymbol R=\text{Norm}(\sigma^2 \boldsymbol I + \boldsymbol L \boldsymbol L^T)$\;

Calculate the Gaussian vector $\boldsymbol{q}=\boldsymbol V \boldsymbol{\zeta}$\;

Apply Gaussian copula to obtain $\boldsymbol u$ as ${u}_i =\Phi_{\boldsymbol R}(q_i), \forall i=1,\ldots,d $\;
\end{algorithm}

\begin{algorithm}
\caption{Sampling Scheme via Binary Mask.}
\label{algo:sample_binary_mask}
\KwIn{Feature vector $\boldsymbol{x} \in \mathbb{R}^{d}$ with its corresponding weight $\boldsymbol{\alpha} \in \mathbb{R}^{d}$, full-rank or low-rank matrix $\boldsymbol L$ and noise level $\sigma^2$, identity matrix $\boldsymbol{I}$, tuning parameter $t$, $\textrm{round}$ denotes the round operator.}
\KwOut{Binary mask vector $\boldsymbol{z} \in \{0,1\}^d$.}
Apply Algorithm~\ref{algo:corr_noise} to obtain correlated noise vector $\boldsymbol u$\;

Compute each element in the logit $\boldsymbol g$ as $g_i = \log \frac{u_i}{1-u_i}, \forall i=1,\ldots,d$\;

Calculate the probability $ \tilde z_i = \frac{1 }{1 + \exp \{- (g_i + \alpha_i) / t\}}, \forall i=1,\ldots,d $\;

Obtain indicator vector $\boldsymbol z$ via discretizing the probability to binary variable $z_i = \textrm{round} (\tilde z_i), \forall i=1,\ldots,d$\;
\end{algorithm}

\subsection{Sampling Scheme via Top-$k$ Ranking}

When the number of features is known, without loss of generality, we let $k := |\mathcal S_a|$ ($k \leq d$). 
We aim to find the top-$k$ features with the most predictive power.
Suppose we have obtained the weights $\{\alpha_{1}, \alpha_{2}, \ldots \alpha_{d} \}$ from weight layer.
A straightforward way to select the $k$ features randomly proportional to its weights, which is also known as the weighted random sampling (WRS)~\citep{xie2019reparameterizable}. This procedure can be realized through the following ways: (i) For each $i \in [d]$, sample $u_i \sim U(0,1)$ independently and compute keys $v_i = u_i^{1 / \alpha_i}$.
(ii) Select $k$ features with the largest keys $v_i$.

Firstly, we can observe that 
$u_i$ are sampled independently in above procedure. We extend this by adding correlations between $u_i$'s. That is, their joint distribution satisfies
\begin{eqnarray}\label{eq:topk:gaussian}
F(u_1, \ldots, u_d) = C_{gaussian}(u_1, \ldots, u_d; \boldsymbol R),
\end{eqnarray}
where correlation matrix $\boldsymbol R$ should be learned through the network. Here, we assume $\boldsymbol R$ admits the structure $\textrm{Norm}(\boldsymbol I + \tau \boldsymbol L \boldsymbol L^T)$ with $\tau$ as a hyper parameter controlling the magnitude of correlations. \textit{Observe to ourselves that certain equivalence of $\tau$ here and $\sigma$ in~\eqref{eq:sigma_noise} by giving $\tau=\frac{1}{\sigma^2}$}.

\begin{figure*}
\centering
   \includegraphics[width=5in]{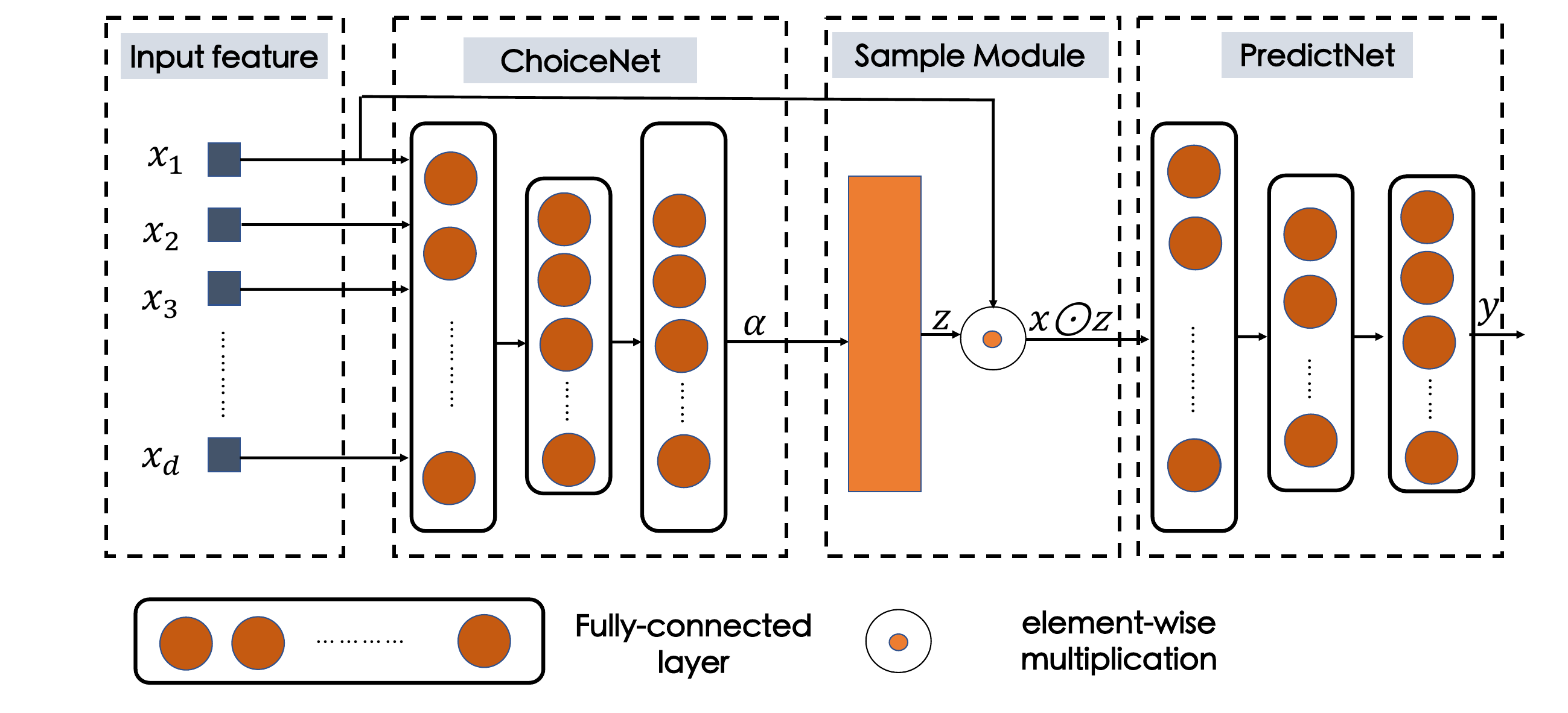}

   \caption{Our feature selection framework, it contains three main parts: ChoiceNet receives the input feature $x_1 \ldots x_d$ and determines the score $\boldsymbol  \alpha$, sampler module then receives $\boldsymbol\alpha$ and outputs the indicator variable $\boldsymbol z$; at last, the suppressed feature  $\boldsymbol x \odot \boldsymbol z$ is fed into PredictNet for prediction. }
   \label{fig:framework}

 \end{figure*}
 
Secondly, suppose the top-$k$ active feature set based on keys $\{v_1, \ldots, v_d\}$ is $\mathcal S_a = \{i_1, \ldots, i_k\}$.
We then define the indicator $z_i = \boldsymbol 1\{i \in \mathcal S_a\}$ to represent whether feature $i$ is selected or not.
By such construction, we know that $z_i$ is not a differentiable function of $\alpha$'s. This may bring difficulty in computing the backward gradients. 
To avoid this problem, we consider a continuous deterministic relaxation to approximate $z_i$'s.
The specific procedure is described as follows.
We define 
\begin{eqnarray}
p_i^s = \frac{\exp\{v_i^s/t\}}{\sum_{l=1}^d \exp\{v_l^s/t\}} \nonumber
\end{eqnarray}
for $s \in \{1, \ldots, k\}$, where $v_j^s$ is defined recursively by 
\begin{eqnarray}
v_i^{s} &=& v_i^{s-1} + t^{\delta} \log(1 - p_i^{s-1})  \quad \textrm{for} ~ s \in \{2, \ldots, k\}; \nonumber \\
v_i^s &=& v_i  \quad \textrm{for} ~ s = 1. \nonumber
\end{eqnarray}
Here $t$ is a tuning parameter that determines the approximation level and $\delta \in [0,1)$ is a hyper-parameter that adjusts the step size.
We then define the relaxed value of $z_i$ which is $\tilde z_i := \sum_{s = 1}^k p_i^s$
and let $\boldsymbol{\tilde z} = (\tilde z_1, \ldots, \tilde z_d), \boldsymbol{z} = (z_1, \ldots, z_d)$.
By writing into the vector form, we have 
\begin{eqnarray}
\boldsymbol{\tilde z} = \sum_{s=1}^k  \boldsymbol p^s, \nonumber
\end{eqnarray}
with $\boldsymbol p^s = (p_1^s, \ldots, p_d^s)$.
Via using this relaxation scheme, our method contains the classical weighted random sampling as special case. The algorithm workflow for top-$k$ ranking method is listed in Algorithm~\ref{algo:sample_topk_ranking}. On top of the above, we offer the following two theorems to shed light on role of $\sigma$ in the convergence of the sampling scheme, with the proofs deferred to the Appendix. 

\input{algo3.tex}

\begin{theorem}
\label{theo:1}
As both $t, \frac{1}{\sigma} \rightarrow 0$, we have $\textrm{Trunc}(\boldsymbol{\tilde z},k) \rightsquigarrow \boldsymbol{z}^{wrs}$.
\end{theorem}
Here, ``$\rightsquigarrow$" represents the convergence in distribution. Function $\textrm{Trunc}(\boldsymbol x,k)$ chooses the indices of largest $k$ elements in $\boldsymbol x$ provided the length of $\boldsymbol{x}$ is at least $k$.
When $1/\sigma \rightarrow 0$, the proposed $\boldsymbol{\tilde z}$ converges to $\boldsymbol{z}^{wrs}$ which follows the weighted sampling distribution.
On the other hand, if we take $\boldsymbol L = (1, \ldots, 1)^T \in \mathbb R^d$ and let $\sigma \rightarrow 0$, then the proposed method can also recover the situation top $k$ sampling based on weights $\{\alpha_1, \ldots, \alpha_d\}$.
\begin{proof}[Proof of Theorem~\ref{theo:1}]
Let $\check{\boldsymbol z}$ be $\textrm{Trunc}(\tilde{\boldsymbol z}, k)$ and it suffices to show that 
\begin{eqnarray}\label{eq:0}
&\lim_{t,\frac{1}{\sigma} \rightarrow 0} \mathbb P(\check{\boldsymbol z} = (i_1, \ldots, i_k)) \\\notag
=& \frac{\alpha_{i_1}}{\sum_{i=1}^d \alpha_i} \cdot \cdot \ldots \cdot \frac{\alpha_{i_k}}{\sum_{i=1}^d \alpha_i - \sum_{j=1}^{k-1}\alpha_{i_j}}.
\end{eqnarray}
By recursive formula of $\boldsymbol p^s$, it can directly verified that for any $\epsilon > 0$, there exists a constant $t_0$ such that
\begin{eqnarray}\label{eq:1}
\tilde z_{i} \geq 1 - \epsilon & & \textrm{for}~ i \in \mathcal S_{v,a}; \nonumber \\
\tilde z_{i} \leq \epsilon & & \textrm{for}~ i \notin \mathcal S_{v,a}
\end{eqnarray}
holds when $t < t_0$. Here, $\mathcal S_{v,a}$ is defined to be the set of indices of $k$ largest keys $v_i$. (Here $v_i :=\log(u_i)/\alpha_i$ is called as the key for $i$-th feature.)
In other words, $\textrm{Trunc}(\tilde{\boldsymbol z}, k)$ returns the indices of top $k$ keys when $t \downarrow 0$. 

On the other hand, when $1 / \sigma \rightarrow 0$, we know that the correlation matrix $\boldsymbol R$ converges to $\boldsymbol I$.
By Fubini's Theorem and following the proof strategy of Proposition 3 in~\citet{efraimidis2006weighted}, we know that 
\begin{eqnarray}\label{eq:2}
\lim_{\sigma \rightarrow 0} \mathbb P(v_1 \leq \ldots \leq v_d) 
= \prod_{i=1}^d \frac{\alpha_i}{\alpha_1 + \ldots + \alpha_i} 
\end{eqnarray}

Combining \eqref{eq:1} and \eqref{eq:2}, it gives exactly \eqref{eq:0}. This completes the proof.
\end{proof}

\begin{theorem}
\label{theo:2}
When $\boldsymbol L = (1, \ldots, 1)^T \in \mathbb R^d$, we have $\textrm{Trunc}(\boldsymbol{\tilde z},k) \rightsquigarrow \boldsymbol{z}^{top k}$ as both $t, \sigma \rightarrow 0 \ (\text  {the same as} ~ \tau \to \infty)$,
where $\boldsymbol{z}^{top k} := (z_{i_1},\ldots, z_{i_k})$.
\end{theorem}

 \begin{proof}[Proof of Theorem~\ref{theo:2}]
Similar to the proof of Theorem~\ref{theo:1}, we again know that $\textrm{Trunc}(\tilde{\boldsymbol z}, k)$ returns the indices of top $k$ keys when $t \downarrow 0$. 
It remains to show that the order of $v_i$'s is the same as the order of $\alpha_i$'s when $1/\sigma\rightarrow 0$.

We only need to show that the probability that, for any pair $i,j$, it holds
\begin{eqnarray}
\lim_{\sigma \rightarrow 0} \mathbb P(v_i \leq v_j) = 0 
\end{eqnarray}
if $\alpha_i > \alpha_j$. By straightforward calculation, we get
\begin{align*}
\mathbb P(v_i \leq v_j)  = \mathbb P(u_i^{1/\alpha_i} \leq u_j^{1/\alpha_j}) \
 = \mathbb P(  \frac{\log u_i}{\log u_j} > \frac{\alpha_i}{\alpha_j} )
\end{align*}
The right hand side goes to 0 as $\sigma \rightarrow 0$. This leads to the desired result.
\end{proof}

\subsection{The Overall Architecture}

We now turn to the detailed implementation via neural networks, which merges the deep neural networks and the proposed framework of feature selection. 

As illustrated in Figure~\ref{fig:framework}, our architecture is composed of three parts: (i) \textbf{ChoiceNet} (abbreviation of Choice Network), which models the mapping $\boldsymbol \alpha=f\left(\boldsymbol{w};\boldsymbol{x}\right)$ and is responsible for selecting features and outputting the learned score $\boldsymbol \alpha $ for each sample. (ii) \textbf{Sampler Module}, which models the sampling function $\mathcal{S}(\boldsymbol \alpha)$ and receives the input $(\boldsymbol \alpha)$ to output $\boldsymbol{z}$, where ${z}_{i}=1$ indicates that the $i$-th feature is preserved or otherwise removed, and the selected feature can be expressed as $\boldsymbol{x} \odot \boldsymbol{z}$. (iii) \textbf{PredictNet} (abbreviation of Predictor Network), which receives the selected feature $\boldsymbol{x} \odot \boldsymbol{z}$ as input and outputs the corresponding prediction. For further details on the practical applications of this architecture in terms of neural networks, please refer to the Appendix.


\subsubsection{Architecture of ChoiceNet and PredictNet}

Both ChoiceNet and CriticNet are three-layer MLPs with ReLU~\citep{brownlee2019gentle} or SeLU~\citep{klambauer2017self} activations. For ChoiceNet, the size of the input layer is $d$, a fully connected layer with the dimension of $h_c$ is added after the input layer, and the dimension of the final layer is also equivalent to $d$. PredictNet differs from ChoiceNet in that the number of units in the last layer is $k$, where $k$ is the number of labels for classification tasks, and $h_p$ is the dimension of the hidden layer. Additionally, a Softmax layer is added to the last layer for prediction, and Batch Normalization is added after the activation function to reduce overfitting.

  \begin{figure*}

  \begin{center}
    \includegraphics[width=5in]{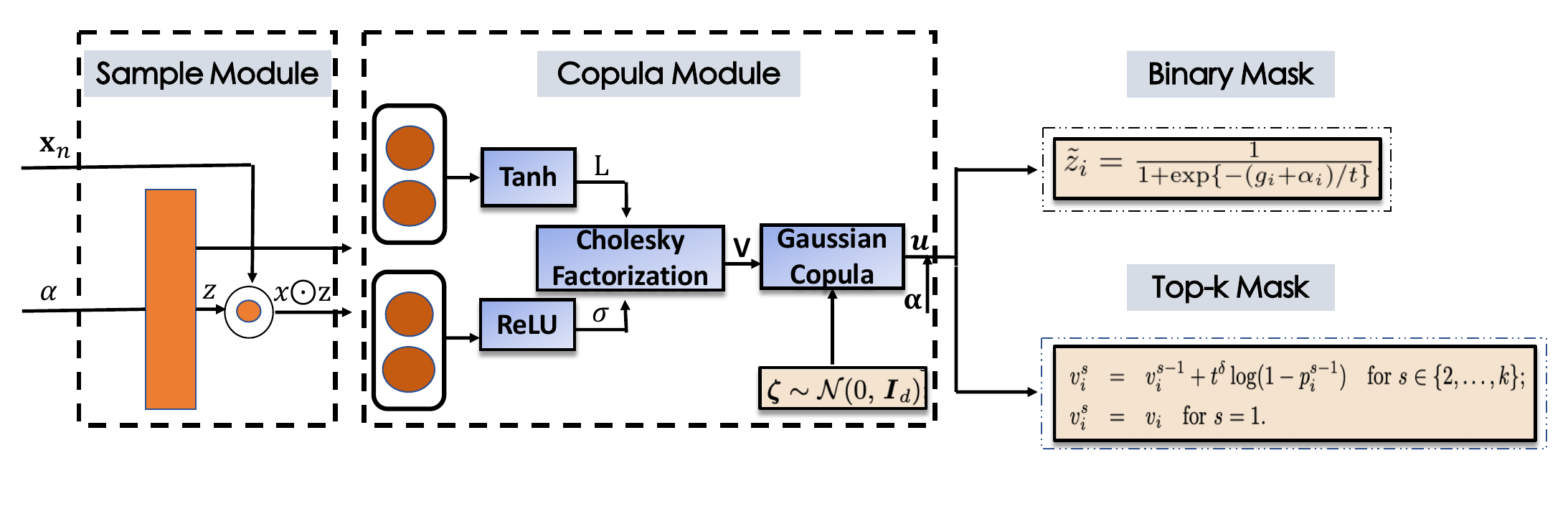}
  \end{center}

\caption{Sampler module $ \mathcal{S}(\boldsymbol \alpha)$. We make correlated uniform noise via Gaussian copula as Algorithm~\ref{algo:corr_noise}, then we apply the two sampling schemes as stated in Algorithm~\ref{algo:sample_binary_mask} and Algorithm~\ref{algo:sample_topk_ranking}.} 
\label{fig:sample_module}

\end{figure*}

\subsubsection{Architecture of Sampler Module}

\paragraph{Architecture of Binary Sampler Module}It is not straightforward to implement a copula using a neural network, as the covariance matrix ${\boldsymbol \Sigma}$ must be guaranteed to be positive semi-definite (PSD). To this end, two extra fully connected layers are included to infer $\boldsymbol L$ and $\sigma$, with tanh and relu activation functions applied after the fully connected layer. After obtaining the covariance function, Cholesky decomposition can be used to obtain the factor matrix $\boldsymbol V$, ensuring that the covariance matrix is PSD. During the inference stage, a Bernoulli distribution is sampled for each sample. Algorithm~\ref{algo:sample_binary_mask} can be leveraged to generate the feature indicator vector, as illustrated in Figure~\ref{fig:sample_module}.

\paragraph{Architecture of Top-$k$ Sampler Module} WRS generates the key $v_i$ for the associated feature $x_i$ as $v_i = u_i^{1/\alpha_i}$. By taking the log transformation, the key can be written as $\tilde{v}_i = \left(1/\alpha_i\right)\log(u_i)$, since the log transformation is monotonic, the induced construction of $\tilde{v}_i$ is still consistent with the original key $v_i$. For concrete realization, we propose to reparameterize the key as a deterministic mapping of the parameters $\left(1/\alpha_i\right)$ and $\log(u_i)$ via a neural network. Then, by leveraging the top-$k$ relaxation in Algorithm~\ref{algo:sample_topk_ranking}, we can build a differentiable approximation with respect to the key $v_i$. The architecture of the top-$k$ sampling scheme in the context of neural networks is similar to that of binary feature selection, except for the block representation. We also provide a demonstration of the block in Figure~\ref{fig:sample_module}.

\section{Experimental Results}

\label{sec:experiments}
In this section, we empirically compare our algorithm with the advanced feature selection methods. For binary feature selection, we quantitatively compare our work with several powerful algorithms including methods based on neural networks like INVASE~\citep{yoon2019invase}, L2X~\citep{chen2018learning}, LIME~\citep{lime}, Shap~\citep{lundberg2017unified}, Knockoff~\citep{barber2015controlling}, and classical methods like Tree~\citep{geurts2006extremely}, SCFS~\citep{hall2000correlation}, LASSO~\citep{tibshirani1996regression}.  For top-$k$ feature selection, we also validate our algorithm with several dominant approaches such as STG~\citep{yamada2020feature}, CAE~\citep{abid2019concrete}, L2X~\citep{chen2018learning}, Shap~\citep{lundberg2017unified}, LASSO~\citep{tibshirani1996regression}, RF~\citep{diaz2006gene}, boosting~\citep{friedman2000additive}. 
In the comparison of classical methods, the advantages of neural network-based methods in an end-to-end fashion are highlighted. \textit{It should be noted that, although these approaches vary in type, they all serve the purpose of feature selection, which is also widely employed in the preceding methods}. See Appendix for the details of implementation and baseline methods.

\begin{table*}[ht]

  \centering
  \caption{Experimental results on six synthetic dataset with dimension 11, we compare with a number of classical and advanced algorithms. Better results are marked in bold. Larger TPR and smaller FDR indicate better results. All the results are extracted from the INVASE method except ours.}

  \resizebox{1.0\linewidth}{!}{
    \begin{tabular}{@{}l|ll|ll|ll|ll|ll|ll@{}}
    \toprule
    Dataset & \multicolumn{2}{c|}{Syn1} & \multicolumn{2}{c|}{Syn2} & \multicolumn{2}{c|}{Syn3} & \multicolumn{2}{c|}{Syn4} & \multicolumn{2}{c|}{Syn5} & \multicolumn{2}{c}{Syn6} \\
    \hline
    Metric & TPR $\uparrow $ & FDR $\downarrow$  & TPR  $\uparrow $  & FDR $\downarrow$  & TPR  $\uparrow $  & FDR $\downarrow$  & TPR $\uparrow $   & FDR   $\downarrow$ & TPR $\uparrow $   & FDR $\downarrow$  & TPR  $\uparrow $   & FDR$\downarrow$ \\
    \rowcolor{lightgray!60}
    Ours  & \textbf{100.0 } & \textbf{0.0 } & 86.4  & 4.8 &96.8  & {2.0 } & { 95.5} & \textbf{2.0} & \textbf{89.3 } & {3.78 } & \textbf{93.8 } & \textbf{6.6 } \\
    
    INVASE~\citep{yoon2019invase} & \textbf{100.0}  &\textbf{0.0}   & \textbf{100.0}  & \textbf{0.0}   & 92.0  & \textbf{0.0}   & \textbf{99.8}  & 10.3  & 84.8  & \textbf{1.1}   & 90.1  & 7.4  \\
    
    L2X ~\citep{chen2018learning} & \textbf{100.0}  & \textbf{0.0}   & \textbf{100.0}  & \textbf{0.0}   & 69.4  & 30.6  & 79.5  & 21.8  & 74.8  & 26.3  & 83.3  & 16.7  \\
    
    Shap~\citep{lundberg2017unified} & 60.4  & 39.6  & 93.3  & 6.7   & 90.9  & 9.1   & 65.2  & 31.9  & 62.9  & 33.7  & 71.2  & 28.8  \\
   
    LIME~\citep{lime}  & 13.8  & 86.2  & \textbf{100.0}  & \textbf{0.0}    & 98.1  & 1.9   & 40.7  & 49.4  & 41.1  & 50.6  & 50.5  & 49.5  \\
   
    Knockoff~\citep{barber2015controlling} & 10.0  & 70.0  & 8.7   & 36.2  & 81.2  & 17.5  & 38.8  & 35.1  & 41.0  & 51.1  & 56.6  & 42.1  \\
 
    Tree~\citep{geurts2006extremely}  & \textbf{100.0}  & \textbf{0.0}  & \textbf{100.0}  & \textbf{0.0}   & \textbf{100.0} & \textbf{0.0}  & 54.7  & 39.0  & 56.8  & 37.5  & 60.0  & 40.0  \\
 
    LASSO~\citep{tibshirani1996regression} & 19.0  & 81.0  & 39.8  & 60.2  & 78.3  & 21.7  & 49.9  & 50.9  & 45.5  & 48.2  & 56.4  & 43.6  \\
  
    SCFS~\citep{hall2000correlation}  & 23.5  & 76.5  & 39.5  & 60.5  & 78.3  & 22.0  & 48.9  & 52.4  & 42.4  & 51.2  & 56.1  & 43.9  \\
    \bottomrule
    \end{tabular}%
    }
  \label{tab:sythetic_11d}

\end{table*}%

\begin{table*}[htbp]
  \centering
  \caption{Experimental results on 100-dimension synthetic datasets. 
  Better results are marked in bold. Larger TPR and smaller FDR indicate better results. All the results are extracted from the INVASE method except ours.}

  \resizebox{1.0\linewidth}{!}{
    \begin{tabular}{@{}l|ll|ll|ll|ll|ll|ll@{}}
    \toprule
    Dataset & \multicolumn{2}{c|}{Syn1} & \multicolumn{2}{c|}{Syn2} & \multicolumn{2}{c|}{Syn3} & \multicolumn{2}{c|}{Syn4} & \multicolumn{2}{c|}{Syn5} & \multicolumn{2}{c}{Syn6} \\
    \hline
    Metric & TPR  $\uparrow $ & FDR$\downarrow$   & TPR$\uparrow $   & FDR $\downarrow$   & TPR$\uparrow $   & FDR $\downarrow$   & TPR $\uparrow $  & FDR$\downarrow$    & TPR $\uparrow $  & FDR  $\downarrow$  & TPR  $\uparrow $ & FDR $\downarrow$\\
     \rowcolor{lightgray!60}
    Ours  & \textbf{100.0 } & \textbf{0.0 } & \textbf{100.0 } & \textbf{0.0 } & \textbf{100.0} & {1.2}   & \textbf{94.5 } & \textbf{40.3}  & \textbf{88.4}  & \textbf{9.4}  & \textbf{99.5 } & \textbf{14.6 } \\

    INVASE~\citep{yoon2019invase} &  \textbf{100.0 } & \textbf{0.0 } & \textbf{100.0 } & \textbf{0.0 }  & \textbf{100.0 } & \textbf{0.0}   & 66.3  & 40.5      & 73.2  & 23.7  & 90.5  & 15.4  \\
 
    L2X~\citep{chen2018learning}   & 6.1   & 93.9  & 81.4  & 18.6  & 57.7  & 42.3  & 48.5  & 46.5  & 35.4  & 60.8  & 66.3  & 33.7  \\
   
    Shap~\citep{lundberg2017unified}  & 4.4   & 85.6  & 95.1  & 4.9   & 88.8  & 11.2  & 50.2  & 43.4  & 49.9  & 44.2  & 62.5  & 37.5  \\
   
    LIME~\citep{lime}  & 0.0   & 100.0  &  \textbf{100.0 }  &  \textbf{0.0 }   & 92.7  & 7.3   & 43.8  & 47.4  & 49.9  & 44.2  & 50.1  & 49.9  \\
   
    Knockoff~\citep{barber2015controlling} & 0.0   & 64.9  & 3.7   & 71.2  & 74.9  & 24.9  & 28.2  & 59.8  & 33.1  & 59.4  & 46.9  & 53.0  \\
   
    Tree~\citep{geurts2006extremely}  & 49.9  & 50.1  & \textbf{100.0 }  &  \textbf{0.0 }   & \textbf{100.0 }  &  \textbf{0.0 }   & 40.7  & 49.5  & 56.7  & 37.5  & 58.4  & 41.6  \\
 
    LASSO~\citep{tibshirani1996regression} & 2.5   & 97.5  & 4.0   & 96.0  & 75.3  & 24.7  & 28.3  & 73.2  & 36.0  & 56.9  & 45.9  & 54.1  \\
  
    SCFS~\citep{hall2000correlation}  & 2.5   & 97.5  & 5.3   & 94.7  & 74.9  & 25.1  & 27.0  & 74.6  & 30.6  & 62.1  & 38.3  & 61.7  \\
    \bottomrule
    \end{tabular}%
    }
  \label{tab:sythetic_100d}%

\end{table*}%

\begin{table*}[h!]
  \centering
  \caption{Experimental results on 100-dimension synthetic datasets with correlated features. Better results are marked in bold. Larger TPR and smaller FDR indicate better results.}

  \resizebox{1.0\linewidth}{!}{
    \begin{tabular}{@{}l|ll|ll|ll|ll|ll|ll@{}}
    \toprule
    Dataset & \multicolumn{2}{c|}{Syn1} & \multicolumn{2}{c|}{Syn2} & \multicolumn{2}{c|}{Syn3} & \multicolumn{2}{c|}{Syn4} & \multicolumn{2}{c|}{Syn5} & \multicolumn{2}{c}{Syn6} \\
\hline
    Metric & TPR $\uparrow $  & FDR$\downarrow$   & TPR$\uparrow $   & FDR$\downarrow$   & TPR $\uparrow $  & FDR$\downarrow$   & TPR $\uparrow $  & FDR$\downarrow$   & TPR $\uparrow $  & FDR$\downarrow$   & TPR $\uparrow $  & FDR$\downarrow$ \\
  \rowcolor{lightgray!60}
    Ours  & \textbf{100.0 } & \textbf{0.0 } & \textbf{100.0 } & \textbf{0.0 } & \textbf{100.0 } & \textbf{0.0 } & \textbf{91.95 } & \textbf{41.6 } & \textbf{90.9 } & \textbf{42.0 } & \textbf{90.1} & \textbf{43.9 } \\

    NOLA  & 74.8  & 18.3  & 96.4  & 16.7  & 99.7  & 78.4  & 48.5  & 84.0  & 66.8  & 49.3  & 50.0  & 46.2  \\

    INVASE~\citep{yoon2019invase} & 69.3  & 27.5  & 93.6  & 14.5  & 91.8  & 72.5  & 62.5  & 87.3  & 46.2  & 53.8  & 37.9  & 41.7  \\
    \bottomrule
    \end{tabular}%
    }
  \label{tab:sythetic_100d_correlation}%
\end{table*}%

We begin our experiments with six challenging synthetic datasets as suggested in INVASE~\citep{yoon2019invase}. $\boldsymbol x$'s are sampled from a multivariate Gaussian distribution with 11 dimensions, and the covariance matrix is an identity matrix. $y$ is dependent on informative and relevant features in $\boldsymbol x$, and is set as a Bernoulli random variable that is 1 with probability $P(y=1|x)=\frac{1}{1+\exp(\gamma)}$. By varying $\gamma$, we can generate different $y$ values across datasets.

\begin{itemize}
    \item \textbf{Syn1:} $\gamma = x_1x_2$
    \item \textbf{Syn2:} $\gamma =  x_1^2 +x_2^2+x_3^2- 4$
    \item \textbf{Syn3:} $\gamma = -10 \times \sin(2x_7) + 2 |x_8| + x_9 + \exp(-x_{10})$
    \item \textbf{Syn4:} if $x_{11} < 0$, $\gamma=x_1x_2$, otherwise, $\gamma=x_3^2+ x_4^2+x_5^2+x_6^2- 4$.
    \item \textbf{Syn5:} if $x_{11}  < 0$, $\gamma=x_1x_2$, otherwise, $\gamma=-10 \times \sin(2x_7) + 2 |x_8| + x_9 + \exp(-x_{10})$.
    \item \textbf{Syn6:}if $x_{11}  < 0$, $\gamma=x_3^2+ x_4^2+x_5^2+x_6^2- 4$, otherwise, $\gamma=-10 \times \sin(2x_7) + 2 |x_8| + x_9 + \exp(-x_{10})$.
\end{itemize}

We generate 10,000 samples for training and 10,000 samples for testing. The response $y$ in the Syn1, Syn2, and Syn3 datasets is determined by the \textbf{identical} subset of features. In contrast, the number of relevant features \textbf{differs} across samples in the Syn4, Syn5, and Syn6 datasets; performance on these datasets can demonstrate the ability of methods to detect instance-wise and population-aware features. We evaluate the performance in terms of true positive rate (TPR) and false discovery rate (FDR).  
As shown in Table~\ref{tab:sythetic_11d}, our method achieves comparable or even better performance than the previous best method INVASE~\citep{yoon2019invase} on Syn1 and Syn3 datasets, demonstrating its capability in selecting relevant features for real data. Moreover, on Syn5 and Syn6 datasets, which are specifically designed for instance-wise feature selection, our algorithm outperforms INVASE~\citep{yoon2019invase} in terms of instance-wise feature selection.

\subsubsection{100-dimensional Synthetic Dataset}

\label{sec:high-dimension}
To further demonstrate the generality of our framework, we experimented on the 100-dimensional synthetic dataset, adding 89 auxiliary, unrelated features and increasing the feature dimension to 100, while still keeping the produced features uncorrelated. As illustrated in Table~\ref{tab:sythetic_100d}, our method remains the best-performing approach across most datasets.

\subsubsection{Correlated Feature Selection}
\label{sec:corre_fs}

 We use the same dataset in Section~\ref{sec:high-dimension}, but with a specially designed covariance matrix $\boldsymbol \Sigma_{i,j}=\frac{1}{2}^{\lvert i-j \rvert}$, where $i$ and $j$ denote the indices of features. To validate whether copula can capture the dependence between features, we compare our method with INVASE, the best competing baseline approach, and a variant of our method, NOLA (\textbf{N}o c\textbf{O}pu\textbf{LA}), where only copula is removed from our framework. The results in Table~\ref{tab:sythetic_100d_correlation} demonstrate that copula performs better than both INVASE and NOLA.

\subsection{Feature Ranking Experiments on Real Datasets}

\label{sec:topk_real_data}

\subsubsection{Dataset Description}
 \label{sec:dataset}
MNIST is a hand-digital dataset comprising 50,000 training samples and 10,000 test samples, drawn from ten classes. Fashion-MNIST is a clothing dataset, containing 60,000 training samples and 10,000 test samples. ISOLET is a speech dataset for predicting which letter-name was spoken, and it includes approximately 8,000 samples with 617 features. We randomly split it into the training set and test set in a 75-25 ratio. The dimensions and sample sizes of the data are summarized in Table~\ref{tab:real_data}.

\begin{table}[ht]
  \centering
  \caption{Details of the real datasets.}
  \scalebox{0.9}{
    \begin{tabular}{llllll}
  \toprule
    Dataset & dimension & size & type & labels \\
    \hline
    MNIST & 784   & 60k & Image & 10 \\
    
    Fashion MNIST & 784   & 70k & Image & 10 \\
   
    ISOLET &  617  & 7.8k  & Audio & 26 \\
    \bottomrule
    \end{tabular}%
 \label{tab:real_data}%
 }

\end{table}%

We now turn to instance-wise top-$k$ feature ranking on real datasets, including MNIST~\citep{lecun1998mnist}, Fashion MNIST~\citep{xiao2017fashion}, and the ISOLET dataset~\citep{cole1990isolet}. For details of the datasets, ablation studies, and visualization results, please refer to the Appendices.

We measure performance by utilizing accuracy metrics and compare each method varying the number of selected features. Figure~\ref{fig:real_data} shows the resulting prediction accuracy, which mostly outperforms other baselines, indicating that our method is a strong candidate in selecting top-$k$ predictive features. The impact of copula and low-rank approximation on final performance is also investigated in the Appendix.

\begin{figure*}[t!]
    \includegraphics[width=2.3in]{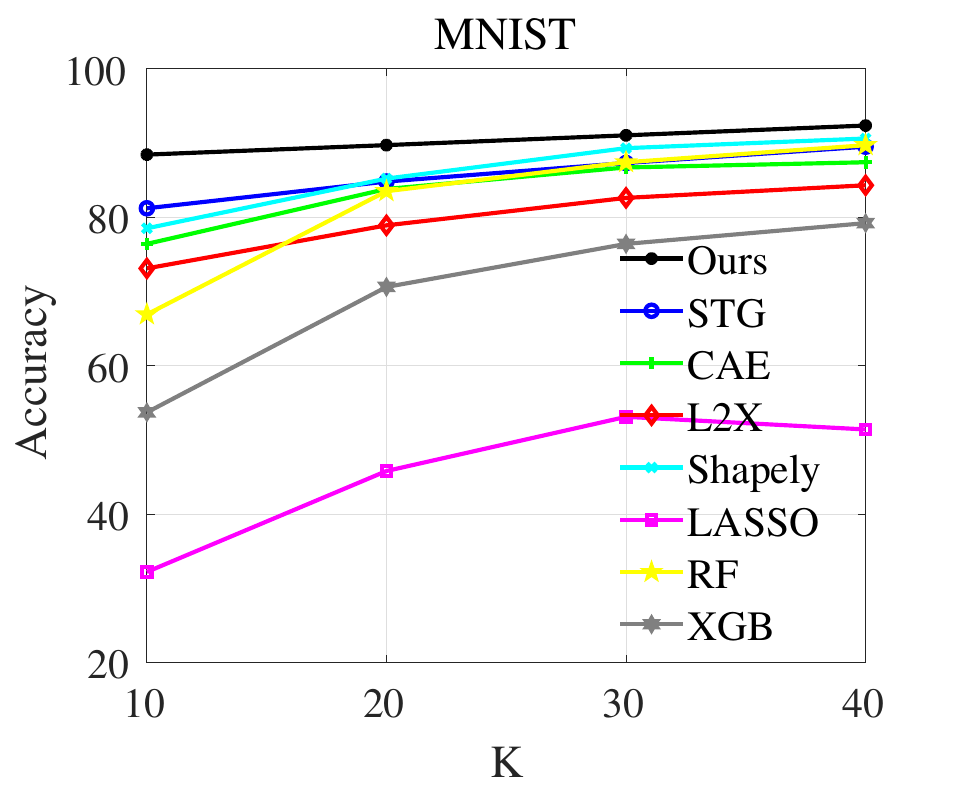}\hspace{-0.20in}
    \includegraphics[width=2.2in]{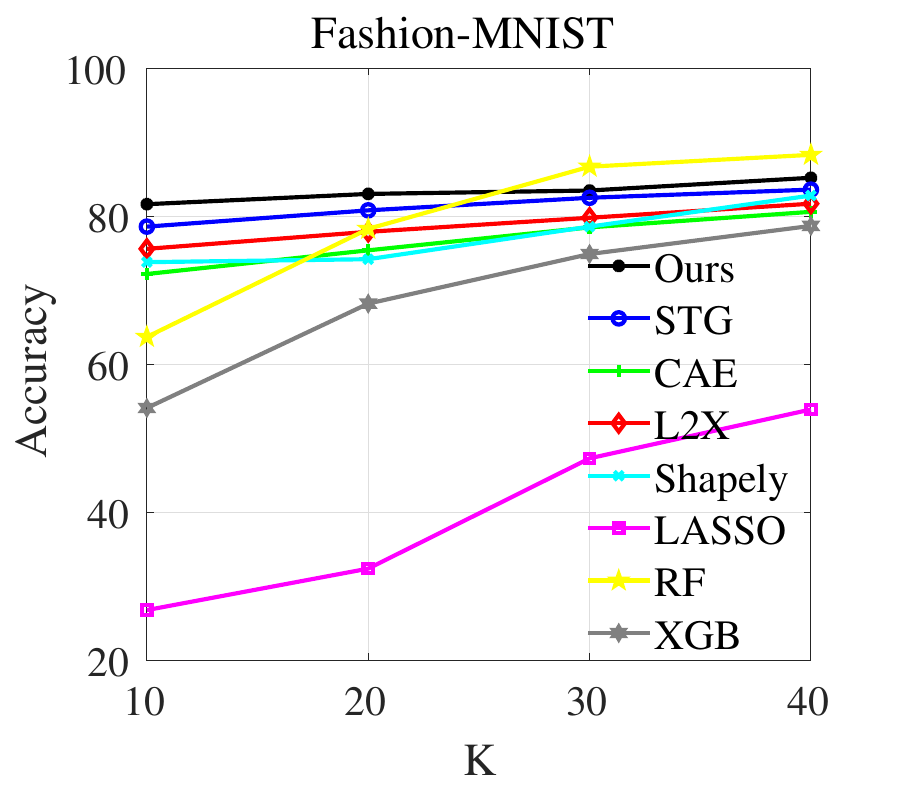}\hspace{-0.20in}
    \includegraphics[width=2.3in]{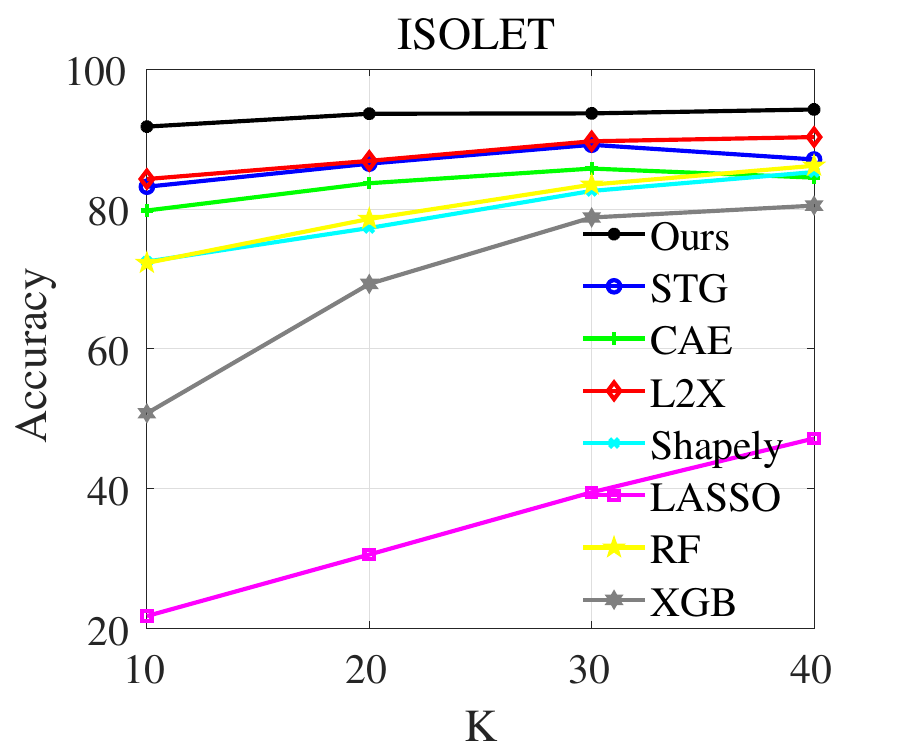}

    \caption{Prediction accuracy vs. the number of selected features $k$ on three real datasets, the value of $k$ varies from 10 to 40. We can discover that our method is superior to most methods. }\label{fig:real_data}

\end{figure*}

\begin{figure*}

{\hspace{-0.4in}    
    \includegraphics[width=7in]{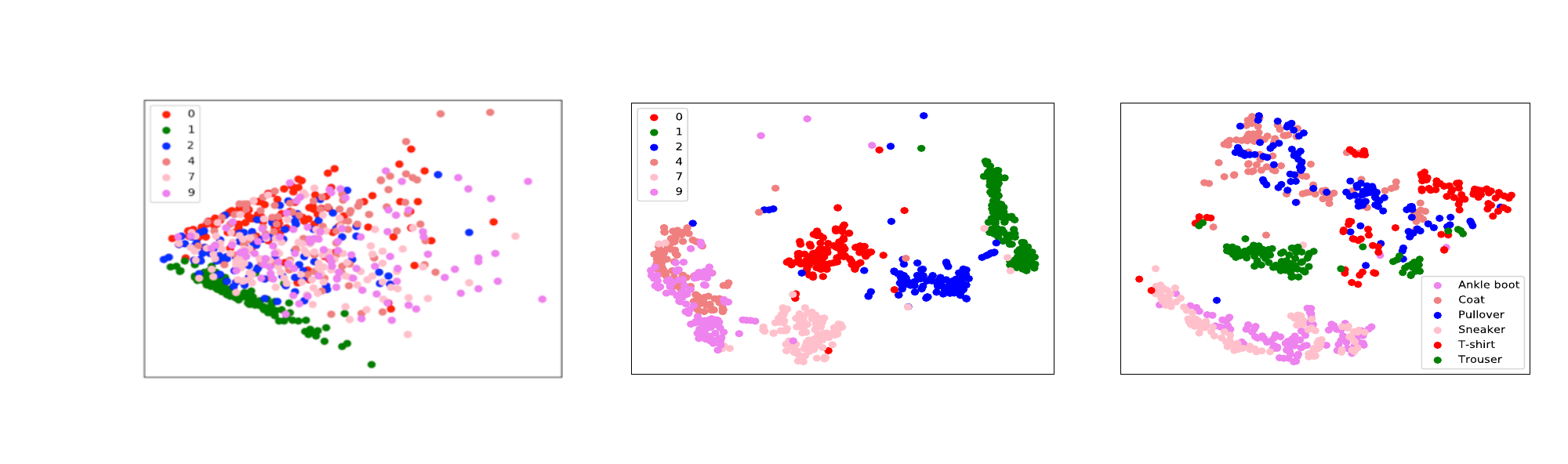}
}

\caption{The left image shows the visualization of induced PSD matrix $\boldsymbol \Sigma \in \mathbb R^{784 \times 784}$ on MNIST dataset for binary feature selection, the middle image illustrates the  $\boldsymbol \Sigma $ for top-$k$ feature ranking on MNIST dataset, right image demonstrates the $\boldsymbol \Sigma$ for top-$k$ feature ranking on Fashion-MNIST dataset.} 
\label{fig:PSD}

\end{figure*}

\subsection{Instance-wise Binary Feature Selection on Synthetic Datasets}

\label{sec:binary_fs}

\subsection{Visualization on PSD Matrix}

To provide an intuitive illustration of the structure of the PSD matrix, we conducted an experiment on the MNIST dataset using copula to perform instance-wise feature selection. We then utilized t-SNE~\citep{van2008visualizing} on the induced PSD matrix to visualize the structure. The left image in Figure~\ref{fig:PSD} provides a demonstration; we can observe that the features with label 0 and label 1 are well separable, since these labels are not similar. Additionally, features with labels 4, 7, and 9 are close to each other, due to their similar appearance. The visualization reflects that the PSD matrix indeed captures the intrinsic character of the feature.

As illustrated in the geometry structure of $\boldsymbol \Sigma$ on the MNIST dataset for binary feature selection, we provide another visualization on the PSD matrix $\boldsymbol \Sigma$ for top-$k$ feature ranking. Following the same profile in Section~\ref{sec:topk_real_data}, we set $k$ to 40 while keeping all other experimental parameters unchanged. The middle image of Figure~\ref{fig:PSD} shows a similar experimental phenomenon and leads to the same conclusion, indicating that the learned $\boldsymbol \Sigma$ can effectively capture the correlated relationship.

We also visualize $\boldsymbol \Sigma$ on the Fashion-MNIST dataset, as shown in the right image of Figure~\ref{fig:PSD}. We observe that the low-dimensional embeddings of Sneaker and Ankle boot are close, as they share similar characteristics. The same conclusion applies to Coat and Pullover, while the embeddings for Trouser are more concentrated and distant from the other image embeddings. This illustrates that copula can effectively capture the relationship between features. Further visualizations can be found in the Appendix.
\subsection{Ablation Studies}
\label{sec:ablation}
\subsubsection{Effects of Low-Rank Approximation} 
We conducted experiments on the MNIST dataset with varying sizes of low-rank approximations. When we performed top-$k$ feature ranking, we set $p$ equal to $k$. The results are presented in Table~\ref{tab:low-rank}, which demonstrate that low-rank approximations can achieve comparable performance to a full-rank scheme. Despite the time-consuming nature of low-rank approximations, whose time complexity is $O(d^3)$, the value of the Gaussian copula cannot be overlooked.

\begin{table}[htbp]
  \centering
  \caption{Comparison on accuracy of low-rank and full-rank approximation on MNIST dataset.}
  \scalebox{0.9}{
    \begin{tabular}{llll}
\toprule
    Dataset & \multicolumn{3}{c}{MNIST} \\
    \hline
    k     & 10    & 20    & 30 \\
   \cline{2-4}
    Full-rank  & \textbf{91.77} & \textbf{92.94} & 93.57 \\

    Low-rank & 91.34 & 92.40  & \textbf{93.93} \\
  \bottomrule
    \end{tabular}%
  \label{tab:low-rank}%
  }

\end{table}%

\subsubsection{Effects of Copula in Top-$k$ Feature Ranking} To investigate the effect of copula on the ultimate performance of top-$k$ feature ranking, we conducted an experiment in accordance with the same protocol outlined in Section~\ref{sec:corre_fs}, wherein only the copula was removed from the framework while all other experimental settings remained the same (hereafter referred to as NOLA). We tested our method on the MNIST dataset, and the results are summarized in Table~\ref{tab:copula_top_k_ranking}. It is evident that our approach significantly enhances the performance when compared to NOLA.

\begin{table}[htbp]
  \centering
  \caption{ Comparison on the prediction accuracy on MNIST dataset with NOLA. }
  \scalebox{0.9}{
    \begin{tabular}{lllll}
   \toprule
    Dataset & \multicolumn{4}{c}{MNIST} \\
    \hline
    k     & 10    & 20    & 30    & 40 \\
  \cline{2-5}
    Ours  & \textbf{91.77} & \textbf{92.94} & \textbf{93.57} & \textbf{93.79} \\
  
    NOLA  & 86.93 & 87.61 & 88.90  & 92.73 \\
  \bottomrule
    \end{tabular}%
  \label{tab:copula_top_k_ranking}%
  }

\end{table}%

\section{Conclusion}

\label{sec:discussion}
In this paper, we have explored the potential of capturing the relationship between correlated features for binary feature selection and top-$k$ feature ranking. To this end, we have successfully incorporated Gaussian copula into the existing feature selection framework with minimal modifications. Our proposed implementation via neural networks has yielded promising results, outperforming many classical and leading methods. We are confident that our work will inspire further research into more effective methods for mining the correlation between features. Possible future directions include capturing tail dependency in features with more sophisticated copulas, or the correlations can be generated via an implicit generative model~\citep{janke2021implicit}.
\appendix

\section{Analysis with respect to complexity}
The computational complexity can be estimated as $O(d^3)$, where $d$ is the number of features. Since we have to do matrix decomposition.
To reduce the computational complexity even further, additional techniques can be applied. As one of the proposed solutions in the article~\cite{lee2022self}, we can threshold the correlation matrix 
$\boldsymbol \Sigma$, and grouping features based on agglomerative clustering using the correlation matrix as the similarity measure. This reduces the number of computations required by conducting block-wise matrix multiplication, which scales quadratically with respect to the largest block size (i.e., the number of features grouped in the largest block). By maintaining the correlation structure within each group, the generated gates for features within the same group can be used to select features that are highly correlated with the target variable, while reducing the computational complexity of the algorithm. If the largest block size remains the same, the complexity of generating the correlated gate vectors will only increase linearly with the feature dimension (since the number of blocks would increase linearly).

The CPU calculation time for copula-based feature selection also depends on the implementation, hardware, and software used. To reduce CPU calculation time, some techniques such as parallelization and approximation methods can be used. For example, using GPU computation can significantly speed up feature selection process, particularly for large datasets.

In summary, copula-based feature selection can be computationally expensive, particularly for large datasets with many features. The computational complexity can be estimated as $O(d^3)$, and the CPU calculation time depends on the implementation, hardware, and software used. However, optimization techniques such as parallelization and approximation methods can be used to reduce the computational complexity and CPU calculation time.

\section{Visualization of Chosen Pixels} 
\label{sec:visual}

\begin{figure}[t!]
\centering
\includegraphics[width=4in]{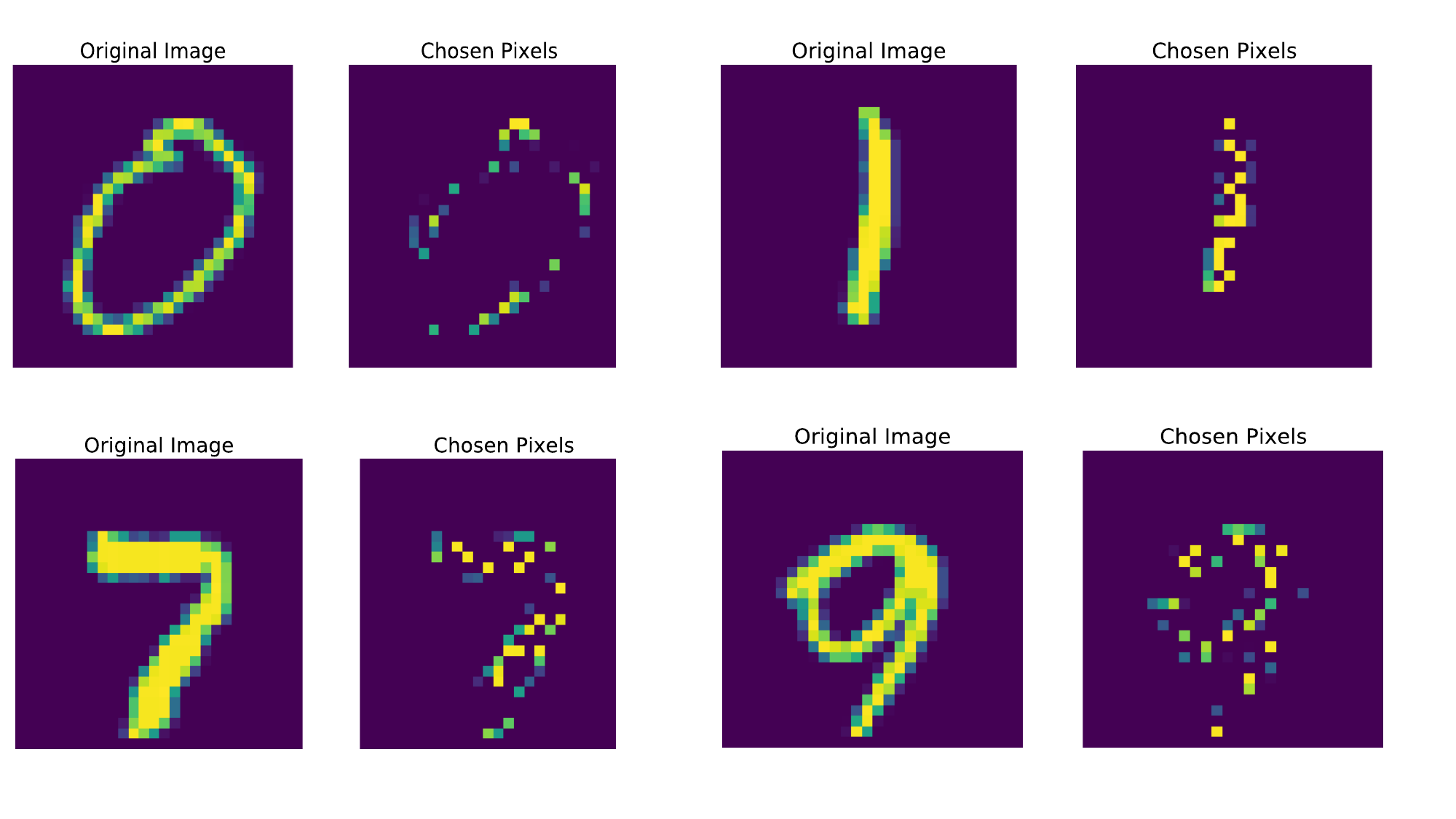}
\caption{Visualization of chosen pixels, we can observe that our method can select meaning and intuitive image features.}
\label{fig:mnist_visual}
\end{figure}

Now, we present a visual representation of the top-120 features with the most significant values in the $\boldsymbol \alpha$ for each sample on the MNIST dataset. For each image, we display the most informative features. As is evident from Figure~\ref{fig:mnist_visual}, we can observe a clear segmentation in the shape of the classification object, indicating that our method is capable of identifying the most important and meaningful features. This visualization demonstrates the remarkable interpretive power of our method.

\section{Neural Reparameterization of Correlated Uniform Noise}
\label{sec:neural_para}

In Algorithm~\ref{algo:neural_corr_noise}, we present the workflow to produce correlated uniform noise $\boldsymbol u \in \mathbb{R}^d$. Here, we offer a further elucidation of the neural reparameterization of it. We denote the output of the hidden layer in ChoiceNet as $f(\boldsymbol{w};\boldsymbol{x})$  as $\boldsymbol{O}_h \in \mathbb{R}^{h_c}$, and the weight parameter of the layer which produces $\sigma$ as $\boldsymbol{W}_{ \sigma}\in \mathbb{R}^{h_c\times d}$. Additionally, $\boldsymbol{W}_{\boldsymbol L}$ stands for the weight parameter of the layer which generates matrix $\boldsymbol L$. If we opt for the low-rank approximation, the shape of $\boldsymbol{W}_{\boldsymbol L}$ is $h_c \times (p\times d)$, otherwise $h_c \times (d\times d)$. The concrete implementation via neural reparameterization is provided as follows:

\begin{algorithm}
\caption{Neural Reparameterization of Correlated Uniform Noise}
\label{algo:neural_corr_noise}
\KwIn{Activation $\boldsymbol{O}_h$ of the hidden layer in ChoiceNet $f(\boldsymbol{w};\boldsymbol{x})$, weight parameters $\boldsymbol{W}_{\sigma}$ and $\boldsymbol{W}_{\boldsymbol L}$}
\KwOut{Correlated uniform noise $\boldsymbol u$.}
$\boldsymbol L =\textrm{ReLU}(\boldsymbol{O}_h\boldsymbol{W}_{\boldsymbol L})$\;

$ \sigma =\textrm{Tanh}(\boldsymbol{O}_h\boldsymbol{W}_{\boldsymbol \sigma})$ \;

Obtain the covariance matrix via low-rank approximation $\boldsymbol \Sigma =\boldsymbol L^T \boldsymbol L + \sigma^2 \boldsymbol I$\ or full-rank approximation $\boldsymbol \Sigma =\boldsymbol L^T \boldsymbol L$\;

Perform Cholesky factorization on $\boldsymbol \Sigma$ to get Cholesky factor $\boldsymbol V$ \;

Generate a Gaussian noise vector $\boldsymbol \zeta$ from standard normal distribution $\boldsymbol{\zeta} \sim \mathcal{N}(0,\,\boldsymbol{I} )$\;

Calculate the Gaussian vector $\boldsymbol{q}=\boldsymbol V \boldsymbol{\zeta}$\;

Apply Gaussian copula to obtain $\boldsymbol u$ as ${u}_i =\Phi_{\boldsymbol R}(q_i), \forall i=1,\ldots,d $\;
\end{algorithm}

\section{Details of Baseline Methods}
\label{sec:baseline}
 The summary of some baseline methods of binary feature selection are as follows.

\begin{itemize}
    \item \textbf{Xgboost} We used the Gini index as the splitting criterion. Specifically, we used the DecisionTreeClassifier function from the scikit-learn library with default parameters. The DecisionTreeClassifier function builds a decision tree by recursively splitting the data based on the feature with the highest Gini importance score. The Gini importance score measures the total reduction of impurity brought by a feature in the decision tree, and features with higher Gini importance scores are considered more important. After building the decision tree, we selected the top-$k$ features with the highest Gini importance scores as the selected features. The value of $k$ was determined using the same experimental setup and evaluation protocol as our proposed method. 

    \textbf{LASSO}  In the paper, LASSO is a linear regression-based feature selection method, where we used $L_1$ regularization to encourage sparsity in the model coefficients. Specifically, we used the LogisticRegression function from the scikit-learn library with $L_1$ penalty and default parameters. The $L_1$ penalty in the LogisticRegression function encourages sparsity in the model coefficients by adding a penalty term to the loss function that is proportional to the absolute value of the coefficients. This penalty term forces the model to select only a subset of the most important features, effectively performing feature selection. After fitting the logistic regression model with $L_1$ regularization, we selected the top-$k$ features with the highest absolute coefficients as the selected features. The value of $k$ was determined using the same experimental setup and evaluation protocol as our proposed method, including a hold-out strategy and 5-fold cross-validation for hyperparameter tuning.
 
    \item \textbf{L2X}~\citep{chen2018learning} introduces a new model interpretation way from the feature selection perspective, it aims to learn a feature selection network that maximizes the mutual information between selected feature subsets and corresponding outputs, we use the official implementation to evaluate the result from the link: 
    
    https://github.com/Jianbo-Lab/L2X
    \item \textbf{INVASE}~\citep{yoon2019invase} proposes an instance-wise feature selection algorithm based on the actor-critic framework to selects most relevant features that minimizes the Kullback-Leibler (KL) divergence between full conditional distribution and suppressed feature distribution, we use the official implementation to evaluate the result from the link:
    
    https://github.com/jsyoon0823/INVASE
    \item \textbf{LIME}~\citep{lime} is a model-agnostic explanation algorithm, it learns an interpretable model locally in a non-redundant and faithful manner by formulating the task as a submodular optimization problem, we use the official implementation to evaluate the result from the link:
    
    https://github.com/marcotcr/lime
    \item \textbf{Shap}~\citep{lundberg2017unified} proposes a novel framework that employs the shapely value to calculate the feature importance, we use the official implementation to evaluate the result from the link:
    
    https://github.com/slundberg/shap
    \item \textbf{Knockoff}~\citep{barber2015controlling} aims to find which variables are important to the response by comparison between knock-off variables and original variables. we use the official implementation to evaluate the result from the link: 
    
    http://web.stanford.edu/group/candes/knockoffs/software/knockoff/
    
  \end{itemize}
In our experiments, we used a neural network as the predictive model in conjunction with Shap and LIME. The neural network had the same size and architecture as the one in our proposed method, in order to ensure fairness and exclude the influence of other factors such as network architecture and size. Regarding the cutoffs, we used a threshold of 0.5 for the neural network to predict the binary class labels. We applied the same threshold when generating the Shap and LIME explanations, to ensure consistency in the interpretation of feature importance across different methods.

   The description of some baseline methods of top-$k$ feature ranking are as follows:
   
   \begin{itemize}
     \item \textbf{STG}~\citep{yamada2020feature} provides a novel algorithm that depends on the Gaussian-based relaxation of the Bernoulli distribution to select relevant features, we use the official implementation to evaluate the result from the link: 
     
     https://github.com/runopti/stg
     \item \textbf{CAE}~\citep{abid2019concrete} introduces an auto-encoder architecture for global feature selection while reconstructing the input, we use the official implementation to evaluate the result from the link: 
     
     https://github.com/mfbalin/Concrete-Autoencoders
 \end{itemize}
 \subsection{Discussions of Baseline Methods}
Our method is capable of discerning pertinent features on a global scale (Syn1, Syn2, and Syn3) as well as on an individual basis (Syn4, Syn5, and Syn6). Notably, our approach surpasses prior neural network-based approaches (INVASE, L2X) in terms of individual performance, with the improvement being more pronounced in Syn4 and Syn5 than in Syn6. Random Forests (RFs) can select global features, thus performing better on Syn1, Syn2, and Syn3, but not as well on Syn4, Syn5, and Syn6. Shapley-based methods calculate the variable importance to elucidate the linear dependency for each sample; however, it is difficult to capture the non-linear relationships in synthetic data, thus rendering it less effective in high-dimensional data. LIME utilizes simple functions to interpret complexity locally; however, it can only explain the particular instance, meaning that it may not be able to accommodate for unseen instances. Knockoff filters features according to certain criteria, yet there is no assurance that this metric is uniquely optimal, thus its performance may vary across datasets.

 \section{Implementation Details}
 \label{sec:impleme}
 \subsection{Synthetic Datasets}
We have employed the same datasets and network structure as in~\citet{yoon2019invase}. For all experiments on the six synthetic datasets, the hyperparameters $h_c$ and $h_p$ were set to 100 and 200, respectively. The activation function of the last layer was a sigmoid. The entire network (ChoiceNet and PredictNet) was trained for 1,000 epochs using Adam with a batch size of 1,000, a weight decay of 0.001, and coefficients of 0.9 and 0.999 for computing running averages of gradient and its square. The constant learning rate was set to 0.0001, and the temperature parameter $t$ was set to either 3 or 5. Finally, cross-validation was employed to tune the hyperparameter $\lambda$.
\subsection{Real Datasets}
When performing top-$k$ feature ranking, $h_c$ and $h_p$ were both set to 16. In all the experiments, we discovered that the Cholesky decomposition to obtain the Cholesky factor matrix $\boldsymbol{L}$ was time-consuming, so we employed the full-rank scheme. We trained the network for 100 epochs using Adam, with coefficients used for computing running averages of gradient set to 0.9 and 0.999. The constant learning rate was set to 0.001 for Fashion-MNIST and MNIST, and 0.0001 for ISOLET. The batch size was set to 1,000 for MNIST and Fashion-MNIST, and 64 for ISOLET, while the temperature $t$ was set to $1$ for all the experiments evaluated on real datasets. All the parameters of the neural networks were randomly initialized. For a fair comparison with baseline methods based on neural networks such as INVASE, L2X, STG, and CAE, we employed the same hyper-parameters and architecture to evaluate the performance.

\bibliography{refs_scholar}

\end{document}

%% file: algo3.tex
 \begin{algorithm}[t]
\caption{Sampling Scheme via Top-$k$ Ranking.}
\label{algo:sample_topk_ranking}
\KwIn{Feature vector $\boldsymbol{x} \in \mathbb{R}^{d}$ with its corresponding weight $\boldsymbol{\alpha} \in \mathbb{R}^{d}$, full-rank or low-rank matrix $\boldsymbol L$ and noise level $\sigma^2$, identity matrix $\boldsymbol{I}$, tuning parameter $t,\delta$. }
\KwOut{Binary mask vector $\boldsymbol{z} \in \{0,1\}^d$.}
Apply Algorithm~\ref{algo:corr_noise} to obtain correlated noise vector $\boldsymbol u$\;

Compute keys $v_i = u_i^{1 / \alpha_i},$ $\forall i=2,\ldots,d$,$v_1 = u_1$;

Take the log transformation $\log(\cdot)$ to $v_i$, and rewrite the key $ v_i=(1 / \alpha_i) \log(u_i)$\;

\For{$s\gets 2$ \KwTo $k$ }{
   ${v}_i^{s} = {v}_i^{s-1} + t^{\delta} \log(1 - p_i^{s-1}), \forall i=1,\ldots,d$ \;
   
   $p_i^s = \frac{\exp\{{v}_i^s/t\}}{\sum_{l=1}^d \exp\{{v}_l^s/t\}}, \forall i=1,\ldots,d$\
    }
$\boldsymbol{\tilde z} = \sum_{s=1}^k  \boldsymbol p^s$\;

Obtain feature indicator vector $\boldsymbol z$ via masking the top-$k$ largest elements in $\boldsymbol {\tilde z}$ with 1, other elements with 0 \;
\end{algorithm}